\documentclass[letterpaper]{article}
\usepackage[margin=1in,dvips]{geometry}
\usepackage{graphicx,psfrag,amsmath,amsthm,amssymb}
\usepackage{natbib}
\usepackage{algorithmic,algorithm}
\usepackage{times}
\usepackage{url}
\usepackage{hyperref}

\newcommand{\A}{\ensuremath{\mathbf{A}}}
\newcommand{\B}{\ensuremath{\mathbf{B}}}
\newcommand{\C}{\ensuremath{\mathbf{C}}}
\newcommand{\D}{\ensuremath{\mathbf{D}}}

\newcommand{\HH}{\ensuremath{\mathbf{H}}}
\newcommand{\I}{\ensuremath{\mathbf{I}}}

\newcommand{\K}{\ensuremath{\mathbf{K}}}
\newcommand{\LL}{\ensuremath{\mathbf{L}}}

\newcommand{\PP}{\ensuremath{\mathbf{P}}}
\newcommand{\Q}{\ensuremath{\mathbf{Q}}}
\newcommand{\RR}{\ensuremath{\mathbf{R}}}
\renewcommand{\SS}{\ensuremath{\mathbf{S}}}
\newcommand{\T}{\ensuremath{\mathbf{T}}}
\newcommand{\U}{\ensuremath{\mathbf{U}}}
\newcommand{\V}{\ensuremath{\mathbf{V}}}
\newcommand{\W}{\ensuremath{\mathbf{W}}}

\renewcommand{\aa}{\ensuremath{\mathbf{a}}}
\renewcommand{\b}{\ensuremath{\mathbf{b}}}
\renewcommand{\c}{\ensuremath{\mathbf{c}}}

\newcommand{\uu}{\ensuremath{\mathbf{u}}}
\newcommand{\vv}{\ensuremath{\mathbf{v}}}
\newcommand{\w}{\ensuremath{\mathbf{w}}}
\newcommand{\x}{\ensuremath{\mathbf{x}}}
\newcommand{\y}{\ensuremath{\mathbf{y}}}

\newcommand{\0}{\ensuremath{\mathbf{0}}}
\newcommand{\1}{\ensuremath{\mathbf{1}}}

\newcommand{\balpha}{\ensuremath{\boldsymbol{\alpha}}}
\newcommand{\bbeta}{\ensuremath{\boldsymbol{\beta}}}

\newcommand{\bPhi}{\ensuremath{\boldsymbol{\Phi}}}

\newcommand{\bSigma}{\ensuremath{\boldsymbol{\Sigma}}}

\newcommand{\bbR}{\ensuremath{\mathbb{R}}}

\newcommand{\calH}{\ensuremath{\mathcal{H}}}

\newcommand{\calN}{\ensuremath{\mathcal{N}}}
\newcommand{\calO}{\ensuremath{\mathcal{O}}}

\newcommand{\abs}[1]{\left\lvert#1\right\rvert}
\newcommand{\norm}[1]{\left\lVert#1\right\rVert}

{%
\begin{list}{#1}{
\vspace{-\topsep}
\vspace{-\partopsep}
\setlength{\itemindent}{0cm}
\setlength{\rightmargin}{0cm}
\setlength{\listparindent}{0cm}
\settowidth{\labelwidth}{#1}
\setlength{\leftmargin}{\labelwidth}
\addtolength{\leftmargin}{\labelsep}
\setlength{\itemsep}{0cm}
}%
}%
{%
\end{list}
\vspace{-\topsep}
\vspace{-\partopsep}
}

{\begin{enumerate}%
}%
{\end{enumerate}}

\newcommand{\traceop}{\operatorname{tr}}
\newcommand{\trace}[1]{\ensuremath{\traceop\left(#1\right)}}

\graphicspath{{grf/}}

\newtheorem{thm}{Theorem}
\newtheorem{lem}{Lemma}

\newcommand{\ie}{i.e.}
\newcommand{\eg}{e.g.}
\newcommand{\pr}[1]{\text{Pr}\left( #1 \right)}

\title{On Column Selection in Approximate Kernel Canonical Correlation Analysis}
\author{Weiran Wang \\
  Toyota Technological Institute at Chicago\\
  6045 S Kenwood Ave., Chicago IL 60637 \\
  \texttt{weiranwang@ttic.edu}}

\begin{document}
\maketitle

\begin{abstract} 
  We study the problem of column selection in large-scale kernel canonical correlation analysis (KCCA) using the Nystr\"om approximation, where one approximates two positive semi-definite kernel matrices using ``landmark'' points from the training set. When building low-rank kernel approximations in KCCA, previous work mostly samples the landmarks uniformly at random from the training set. We propose novel strategies for sampling the landmarks non-uniformly based on a version of statistical leverage scores recently developed for kernel ridge regression. We study the approximation accuracy of the proposed non-uniform sampling strategy, develop an incremental algorithm that explores the path of approximation ranks and facilitates efficient model selection, and derive the kernel stability of out-of-sample mapping for our method. Experimental results on both synthetic and real-world datasets demonstrate the promise of our method.
\end{abstract} 

\section{Introduction}
\label{s:intro}

Kernel canonical correlation analysis (KCCA, \citealp{LaiFyfe00a,Akaho01a,Melzer_01a,BachJordan02a,Hardoon_04a}) is the kernel extension of the classic canonical correlation analysis (CCA, \citealp{Hotell36a}) algorithm, and has found wide application for analyzing multi-view data in different areas, such as computer vision~\citep{SocherLi10a}, speech recognition~\citep{Rudzic10a,AroraLivesc12a}, natural language processing~\citep{Vinokour_03a,LiShawe-Taylor05a,Hardoon_07a,Hodosh_13a}, computational biology~\citep{Yamanis_04a,Hardoon_07a,Blasch_11a}, and statistics~\citep{BachJordan02a,Fukumiz_07a,Lopez_13a}.

The main idea of CCA is to \emph{linearly} project two random vectors into a lower-dimensional subspace so that the projections are maximally correlated. To extend CCA to nonlinear mappings with greater representation power, KCCA first maps the input observations of each view into Reproducing Kernel Hilbert Spaces (RKHS) and then applies linear CCA in the RKHS. Formally, given a dataset of $N$ pairs of observations ${(\x_1,\y_1),\dots,(\x_N,\y_N)}$ of the random vectors $(\x,\y)$, where $\x_i\in \bbR^{d_x}$ and $\y_i\in \bbR^{d_y}$ for $i=1,\dots,N$, KCCA maps the inputs into $\phi(\x_i)=k_x(\x_i,\cdot)\in \calH_x$ and $\psi(\y_i)=k_y(\y_i,\cdot)\in \calH_y$, where $k_x$ and $k_y$ are positive definite kernels associated with RKHS spaces $\calH_x$ and $\calH_y$ respectively, and then solves the following \emph{regularized} objective
\begin{align}\label{e:kcca}
& \max\limits_{f \in \calH_x,g \in \calH_y} \left<f, \bSigma_{xy} g\right>_{\calH_x} \\
\text{s.t.}\quad & \left<f, \bSigma_{xx} f\right>_{\calH_x} + 2 \lambda_1 \norm{f}^2_{\calH_x}=1 \nonumber \\
& \left<g, \bSigma_{yy} g\right>_{\calH_y} + 2 \lambda_2 \norm{g}^2_{\calH_y}=1 \nonumber
\end{align}
where $(\lambda_1,\lambda_2)>0$ are regularization parameters that help avoid trivial solutions (and enable matrix inversions), and the cross- and auto- covariance operators are defined as
\begin{align*}
\left<f, \bSigma_{xy} g\right>_{\calH_x}&=\frac{1}{N} \sum\nolimits_{i=1}^N \left<f, \bar{\phi}(\x_i)\right>_{\calH_x} \left<g, \bar{\psi}(\y_i)\right>_{\calH_y},\\
\left<f, \bSigma_{xx} f\right>_{\calH_x}&=\frac{1}{N} \sum\nolimits_{i=1}^N \left<f, \bar{\phi}(\x_i)\right>_{\calH_x} \left<f, \bar{\phi}(\x_i)\right>_{\calH_x},
\end{align*}
with $\bar{\phi}(\x_i) = \phi(\x_i) - \frac{1}{N} \sum_{j=1}^N \phi(\x_j)$ and $\bar{\psi}(\y_i)= \psi(\y_i)-\frac{1}{N} \sum_{j=1}^N \psi(\y_j)$. From now on, we denote by $\K_1$ and $\K_2$ the \emph{uncentered} kernel (Gram) matrices, \ie, $(\K_1)_{ij}=\left<{\phi}(\x_i), {\phi}(\x_j)\right>$ and $(\K_2)_{ij}=\left<{\psi}(\y_i), {\psi}(\y_j)\right>$, and denote by $\overline{\K}_1 = \HH \K_1 \HH$ and $\overline{\K}_2  = \HH \K_2 \HH$ their {centered} version, where $\HH = \I- \frac{1}{N} \1 \1^\top \in \bbR^{N \times N}$.

Following the derivation of \citet{BachJordan02a}, it suffices to consider $f$ and $g$ that lie in the span of $\bar{\phi}(\x_i)$ and $\bar{\psi}(\y_i)$ respectively, i.e., $f=\sum_{i=1}^N \alpha_i \bar{\phi}(\x_i)$ and $g=\sum_{i=1}^N \beta_i \bar{\psi}(\y_i)$, and the optimal coefficients $\balpha=[\alpha_1,\dots,\alpha_N]^\top$,\ \ $\bbeta=[\beta_1,\dots,\beta_N]^\top  \in \bbR^N$ satisfy the following eigenvalue system\footnote{Note that \citet{BachJordan02a} used approximations of the form $\left<f, \bSigma_{xx} f\right>_{\calH_x}+2 \lambda_1 \norm{f}^2_{\calH_x}=\frac{1}{N} \balpha^\top \overline{\K}_1^2 \balpha + 2\lambda_1 \balpha^\top \overline{\K}_1 \balpha \approx \frac{1}{N} \balpha^\top (\overline{\K}_1 + N \lambda_1 \I)^2 \balpha$, which still lead to a consistent estimate of the kernel canonical correlation.}
\begin{align}\label{e:kcca-eig}
\left[
\begin{array}{cc}
\0 & \T \\
\T^\top & \0
\end{array}
\right]
\left[
\begin{array}{c}
\balpha' \\ \bbeta'
\end{array}
\right]
= \rho 
\left[
\begin{array}{c}
\balpha'\\ \bbeta'
\end{array}
\right],
\end{align}
where $\balpha'=\frac{1}{\sqrt{N}} (\overline{\K}_1 + N \lambda_1 \I) \balpha$,\ \ $\bbeta'=\frac{1}{\sqrt{N}} (\overline{\K}_2 + N \lambda_2 \I) \bbeta$ with $\norm{\balpha'}=\norm{\bbeta'}=1$, and
\begin{gather} \label{e:def-T}
\T=(\overline{\K}_1 + N \lambda_1 \I)^{-1} \overline{\K}_1 \overline{\K}_2 (\overline{\K}_2 + N \lambda_2 \I)^{-1},
\end{gather}
and $\rho$ is the largest eigenvalue of the system and the optimal objective of \eqref{e:kcca}. Alternatively, $\rho$ is the largest singular value of $\T$, while $\balpha'$ and $\bbeta'$ are the corresponding left and right singular vectors. The view 1 projection mapping is 
\begin{gather} \label{e:out-of-sample}
f(\x)=\sum_{i=1}^N \alpha_i \left< \bar{\phi}(\x_i), \bar{\phi}(\x) \right> \quad \in \bbR.
\end{gather} 
$L$-dimensional projections can be obtained by simultaneously extracting the top $L$ eigenvectors of \eqref{e:kcca-eig}.

Although KCCA has a closed-form solution, for applications where the training set size $N$ is large, it is computationally infeasible to evaluate and store the $N\times N$ kernel matrices, let alone compute $\T$ and its rank-$L$ singular value decomposition (SVD) which are of time complexity $\calO(N^3)$ and $\calO(L N^2)$ respectively. As a result, various low-rank kernel approximation techniques have been proposed to alleviate these issues, including incomplete Cholesky decomposition~\citep{BachJordan02a}, partial Gram-Schmidt~\citep{Hardoon_04a}, incremental SVD~\citep{AroraLivesc12a}, and random Fourier features~\citep{Lopez_14b}. Another popular low-rank kernel approximation method is the Nystr\"om approximation~\citep{WilliamSeeger01a,Fowlkes_04a,Kumar_12a,GittenMahoney13a}: By sampling a set of $M$ landmark points from the training set with index set $I=\{i_1,\dots,i_M\}$, a symmetric positive semi-definite (PSD) matrix $\K \in \bbR^{N\times N}$ is approximated as
\begin{gather}\label{e:standard-nystrom}
\K \approx \LL:=\C \W^{\dagger} \C^\top
\end{gather}
where $\C\in \bbR^{N\times M}$ contains the sampled columns of $\K$ indexed by $I$, and $\W\in \bbR^{M\times M}$ is the square sub-matrix of $\K$ with rows and columns both indexed by $I$.

One can build Nystr\"om approximations $\LL_1$ of rank $M_1$ for $\K_1$, and $\LL_2$ of rank $M_2$ for $\K_2$ in KCCA; we denote the resulting algorithm NKCCA. NKCCA typically outperforms approximate KCCA using random Fourier features~\citep{Lopez_14b} for the same rank, even when a naive uniform sampling strategy is adopted.\footnote{This work uses $\C (\W^\dagger)^{1/2}\in \bbR^{N\times M}$ as the new data matrix in each view and run linear CCA on top. One can show this is equivalent to solving an eigenvalue system similar to \eqref{e:kcca-eig} with Nystr\"om kernel approximations, using the argument in \citet[Sec.~4]{Lopez_14b}.}

\paragraph{Our contributions} 
First, we propose a non-uniform sampling strategy for NKCCA, and prove its approximation guarantee in terms of the kernel canonical correlation. We show that, by carefully selecting small fractions of columns of $\K_1$ and $\K_2$ in Nystr\"om approximations, we obtain an NKCCA solution that is an accurate estimate of the solution to exact KCCA, and that to achieve the same accuracy, the proposed strategy may require fewer columns to be sampled than does uniform sampling. This is the first approximation guarantee for NKCCA that we are aware of, despite its common use in practice~\citep{Lopez_14b,Xie_15a,Wang_15b}. Our strategy is  motivated by that of \citet{AlaouiMahoney15a} for kernel ridge regression, who recently showed that if the columns are sampled according to a distribution depending on a version of statistical leverage scores, the statistical performance can be guaranteed when low-rank kernel approximations are used. Notice that CCA is closely related to regression, as its objective is to find a common subspace where the two views are maximally correlated and predictive of each other. %% Furthermore, the kernel ridge regressor has the form $\K(\K+n\lambda \I)^{-1}\y$ which resembles that of matrix $\T$ in the KCCA solution. 
Second, motivated by the algorithm of \citet{Rudi_15a} for kernel ridge regression, we propose an incremental Nystr\"om approximation algorithm for NKCCA, which explores the entire path of approximation ranks and facilitates efficient model selection. While the first two points are motivated by recent work, the column selection problem here is more interesting and challenging since there are two kernel matrices in KCCA.
Third, we derive the kernel stability of NKCCA, \ie, the perturbation of the projection mapping due to kernel approximations used in training, and our approach can be extended to other approximate KCCA algorithms with spectral error bounds in $\T$. As far as we know, this is the first kernel stability result for KCCA.

%% In the following sections we provide our theoretical and algorithmical results (Sec.~\ref{s:algorithm}), discuss related work (Sec.~\ref{s:related}), and give empirical demonstrations (Sec.~\ref{s:expt}).

\section{Column selection in KCCA}
\label{s:algorithm}

\paragraph{Notations} Boldface capital letters denote matrices; boldface lower-case letters denote column vectors; without boldface, lower-case letters denote scalars. $\1$ denotes the column vector of all ones, and $\I$ denotes the identity matrix with appropriate dimensions. $[\uu; \vv]$ denotes the concatenation of two vectors $\uu$ and $\vv$.  For any matrix $\A$, we use $\sigma_j(\A)$ to denote its $j$-th largest singular value,  $\norm{\A}=\sigma_1(\A)$ its spectral norm, and $\A^\dagger$ its pseudo-inverse. We use $\preceq$ to indicate PSD ordering, i.e., for two PSD matrices $\A$ and $\B$,  $\A \preceq \B$ if and only if $\B-\A$ is PSD.

We consider the following rank-$M_1$ approximation to the \emph{uncentered} view 1 kernel matrix
\begin{align}\label{e:L1}
  \K_1 \approx \LL_{\gamma_1}:=\K_1\SS_1 (\SS_1^\top \K_1 \SS_1 + N\gamma_1 \I)^{\dagger} \SS_1^\top \K_1, % \label{e:L2}
  % \K_2 \approx \LL_{\gamma_2}:=\K_2\SS_2 (\SS_2^\top \K_2 \SS_2 + N\gamma_2 \I)^{\dagger} \SS_2^\top \K_2,
\end{align}
where $\gamma_1 \ge 0$, and $\SS_1 \in \bbR^{N \times M_1}$ is the sampling matrix such that $(\SS_1)_{im}>0$ if column $i$ is chosen at the $m$-th trial of sampling the $M_1$ columns of $\K_1$, and zero otherwise. The centered kernel matrix is then approximated as $\overline{\K}_1 \approx \overline{\LL}_{\gamma_1} := \HH \LL_{\gamma_1} \HH$. For the view 2 kernel $\K_2$, we have the rank $M_2$, $\SS_2\in\bbR^{N \times M_2}$, $\LL_{\gamma_2}$, and $\overline{\LL}_{\gamma_2}$ defined analogously. Notice that \eqref{e:L1} is a slight generalization of the standard Nystr\"om approximation; by setting $\gamma_1=\gamma_2=0$ and the nonzero entries of $\SS_1$ and $\SS_2$ to $1$, we recover the scheme in~\eqref{e:standard-nystrom}. We denote the version of $\LL_{\gamma_1}$ (resp. $\overline{\LL}_{\gamma_1}$) with $\gamma_1=0$ as $\LL_1$ (resp. $\overline{\LL}_1$), and similarly for view 2. We 
omit subscripts $1/2$ below if the result holds for both views.

\citet[Lemma~1]{AlaouiMahoney15a} provide the following basic deterministic characterizations of %% the low-rank approximation in \eqref{e:L1}.
$\LL_{\gamma}$.
\begin{lem} \label{lem:psd-ordering}
  Let $\K=\U \bSigma \U^\top$ where $\U$ is orthogonal and $\bSigma$ diagonal non-negative. For $\gamma>0$, we have
  \begin{gather}
    \LL_{\gamma} \preceq \LL \preceq \K.
  \end{gather}
  Moreover, let 
  \begin{gather} \label{e:def-D}
    \D=\bPhi - \bPhi^{1/2} \U^\top \SS \SS^\top \U \bPhi^{1/2}
  \end{gather}
  with $\bPhi=\bSigma (\bSigma + N \gamma \I)^{-1}$. If $\norm{\D} \le t$ for $t \in (0,1)$, then
  \begin{gather}
    \0 \preceq \K - \LL_{\gamma} % \preceq \frac{N\gamma}{1-t}  \K (\K + N \gamma \I)^{-1} 
    \preceq \frac{N\gamma}{1-t} \I.
  \end{gather}
\end{lem}
%% Note that, in the above lemma, $\D$ and thus $t$ depends on $\gamma$.
Based on Lemma~\ref{lem:psd-ordering}, we obtain the following key lemma (see its proof in the supplementary materials).
\begin{lem} \label{lem:error}
  Let $\gamma>0$, and assume the conditions of Lemma~\ref{lem:psd-ordering} hold. Then we have
  \begin{align*}
    \max \left\{ \norm{\K (\K + N \lambda \I)^{-1} - \LL (\LL + N \lambda \I)^{-1} },\ 
      \norm{\K (\K + N \lambda \I)^{-1} - \LL_{\gamma} (\LL_{\gamma} + N \lambda \I)^{-1} } \right\} \le \frac{\gamma/\lambda}{ 1-t }.
  \end{align*}
  Moreover, the same bound also holds if $(\K,\LL_{\gamma},\LL)$ are replaced by $(\overline{\K}, \overline{\LL}_\gamma, \overline{\LL})$ simultaneously.
\end{lem}

Guaranteeing that the conditions in the lemma hold boils down to controlling the spectral norm of $\D$ defined in~\eqref{e:def-D}. Observe that $\D$ is the error of approximating $\bPhi=\bSigma (\bSigma + \gamma \I)^{-1}$ or equivalently the multiplication $(\U \bPhi^{1/2})^\top \cdot (\U \bPhi^{1/2})$ by a subset of the rows of $\U \bPhi^{1/2}$ (indexed by nonzero entries of $\SS$), and the optimal strategy (in Frobenius norm) is to sample each row with a probability proportional to its squared length~\citep{Drineas_06a}. This motivates the following definition of \emph{$\gamma$-ridge leverage scores}~\citep{AlaouiMahoney15a}: For $\gamma>0$, the $\gamma$-ridge leverage scores associated with $\K$ and the parameter $\gamma$ are
\begin{gather}\label{e:leverage-scores}
  l_i (\K,\gamma) = \sum_{j=1}^N \frac{\sigma_j (\K)}{\sigma_j (\K) + N \gamma} \U_{ij}^2, \quad j=1,\dots,N.
\end{gather}
Notice that $l_i (\K,\gamma)$ is precisely the squared norm of the $i$-th row of $\U \bPhi^{1/2}$, and equivalently $\left( \K ( \K + N \gamma \I)^{-1} \right)_{ii}$. 
Moreover, the \emph{effective dimensionality} of $\K$ with parameter $\gamma$ is defined as the sum of $\gamma$-ridge leverage scores: 
\begin{gather}\label{e:deff}
  d_{\text{eff}}(\K,\gamma)=\sum_{i=1}^N l_i(\K,\gamma)=\trace{\K ( \K + N \gamma \I )^{-1}}.
\end{gather}
The following lemma provides the number of rows that need to be sampled for $\norm{\D}$ to be small. It is essentially an application of Theorem~2 of \citet{AlaouiMahoney15a} to the matrix multiplication $(\U \bPhi^{1/2})^\top \cdot (\U \bPhi^{1/2})$.
\begin{lem}\label{lem:sampling}
  Let $M\le N$ and $I=\{i_1,\dots,i_M\}$ be a subset of $\{1,\dots,N\}$ formed by $M$ elements chosen randomly with replacement, according to the distribution
  \begin{gather} \label{e:sampling}
    \pr{choose\; i}=p_i \ge \beta\frac{l_i(\K,\gamma)}{d_{\text{eff}}(\K,\gamma)}, \; i=1,\dots,N,
  \end{gather}
  for some $\beta\in (0,1]$. Let $\SS$ be the corresponding sampling matrix such that $\SS_{{i_j}j}=1/\sqrt{M p_{i_j}}$ for $j=1,\dots,M$ and zero otherwise. If $M \ge 2 \left( \frac{d_{\text{eff}}(\K,\gamma)}{\beta t^2}+\frac{1}{3 t} \right) \log\left( \frac{N}{\delta} \right)$, then with probability at least $1-\delta$, we have $\norm{\D} < t$.
\end{lem}

\paragraph{Remark} Lemma~\ref{lem:sampling} allows the sampling distribution to be different from the optimal one by a factor of $\beta$, at the cost of slight over-sampling by a factor of roughly $1/\beta$ (assuming the first term in the lower bound of $M$ is dominant). If the $\gamma$-ridge leverage scores are very non-uniform but we insist on using the uniform sampling strategy ($p_1=\dots=p_N=\frac{1}{N}$), the over-sampling rate is roughly $\frac{N l_1 (\K,\gamma)}{d_{\text{eff}}(\K,\gamma)}$, which can be very large if $\K$ has a fast decaying spectrum. This shows the advantage of the more data-dependent sampling strategies. 
\citet{AlaouiMahoney15a} also provide an efficient procedure for estimating $\gamma$-ridge leverage scores. 

\subsection{Approximation error of NKCCA}

% In NKCCA, we conduct column selection and construct Nystr\"om approximations $\LL_{1}$ and $\LL_{2}$ indepdently for the two views. 
Our goal is to quantify the perturbation in the kernel canonical correlation $\rho$ when $\K_1$ and $\K_2$ are replaced by $\LL_{1}$ and $\LL_{2}$ respectively. 
Now define 
\begin{align}\label{e:T-low-rank}
  \tilde{\T}=(\overline{\LL}_{1}+N\lambda_1 \I)^{-1} \overline{\LL}_{1} \overline{\LL}_{2} (\overline{\LL}_{2}+N\lambda_2 \I)^{-1},
\end{align}
and denote by $\tilde{\rho}$ its largest singular value with corresponding left and right singular vectors $\tilde{\balpha}'$ and $\tilde{\bbeta}'$ respectively. According to Weyl's inequality~\citep{HornJohnson86a}, we have
\begin{align}
  \abs{\rho-\tilde{\rho}} & \le  \norm{\T - \tilde{\T}} \nonumber \\
& \le \left\lVert (\overline{\LL}_{1} + N \lambda_1 \I)^{-1} \overline{\LL}_{1} \overline{\LL}_{2} (\overline{\LL}_{2} + N \lambda_2 \I)^{-1}  - (\overline{\K}_1 + N \lambda_1 \I)^{-1} \overline{\K}_1 \overline{\LL}_{2} (\overline{\LL}_{2} + N \lambda_2 \I)^{-1} \right\rVert  \nonumber \\ 
  & \quad + \left\lVert (\overline{\K}_1 + N \lambda_1 \I)^{-1} \overline{\K}_1 \overline{\LL}_{2} (\overline{\LL}_{2} + N \lambda_2 \I)^{-1} - (\overline{\K}_1 + N \lambda_1 \I)^{-1} \overline{\K}_1 \overline{\K}_2 (\overline{\K}_2 + N \lambda_2 \I)^{-1}  \right\rVert \nonumber \\ \label{e:error-two-views}
  & \le \norm{ (\overline{\LL}_{1} + N \lambda_1 \I)^{-1} \overline{\LL}_{1}  - (\overline{\K}_1 + N \lambda_1 \I)^{-1} \overline{\K}_1} + \norm{ \overline{\LL}_{2} (\overline{\LL}_{2} + N \lambda_2 \I)^{-1} -  \overline{\K}_2 (\overline{\K}_2 + N \lambda_2 \I)^{-1} }
\end{align}
where we have used the triangle inequality, and the fact that $\norm{\overline{\K}(\overline{\K}+N\lambda \I)^{-1}} \le 1$ in the two inequalities. It is then straightforward to bound each of the two terms using Lemma~\ref{lem:error} and obtain the following guarantee.
\begin{thm}\label{thm:guarantee}
  Assume that, for constructing Nystr\"om approximations $\LL_1$ and $\LL_2$, we sample $M_1$ columns from $\K_1$ and $M_2$ columns from $\K_2$ according to the distributions in \eqref{e:sampling} using ridge leverage scores  $\{l_i (\K_1, \epsilon \lambda_1  (1-t_1) / 2) \}$ and $\{l_i (\K_2, \epsilon \lambda_2 (1-t_2) / 2) \}$ for some $t_1,t_2 \in [0,1)$ respectively. If $M_1 \ge 2 \left( \frac{d_{\text{eff}}(\K_1, \epsilon \lambda_1 (1-t_1) / 2)}{\beta t_1^2}+\frac{1}{3 t_1} \right) \log\left( \frac{2N}{\delta} \right)$ and $M_2 \ge 2 \left( \frac{d_{\text{eff}}(\K_2, \epsilon \lambda_2 (1-t_2) / 2)}{\beta t_2^2}+\frac{1}{3 t_2} \right) \log\left( \frac{2N}{\delta} \right)$, then with probability at least $1-\delta$, we have $\abs{\rho-\tilde{\rho}} \le \epsilon$.
\end{thm}
\begin{proof} Setting $\gamma_1 = \epsilon \lambda_1  (1-t_1)/2$ in Lemma~\ref{lem:error} and using the stated $M_1$ in Lemma~\ref{lem:sampling}, we obtain that with probability at least $1-\delta/2$, 
\begin{align*}
\norm{ (\overline{\LL}_{1} + N \lambda_1 \I)^{-1} \overline{\LL}_{1}  - (\overline{\K}_1 + N \lambda_1 \I)^{-1} \overline{\K}_1}\le \epsilon / 2.
\end{align*}
The same result can be obtained for view 2 and the theorem follows from~\eqref{e:error-two-views} and  a union bound.
\end{proof}

\paragraph{Remarks} 1. It is important to note that, although our analysis largely depends on  $\LL_{\gamma}$ with $\gamma>0$ (\eg, Lemma~\ref{lem:psd-ordering} \&~\ref{lem:error}), it is not used in the actual algorithm. In fact, in Theorem~\ref{thm:guarantee}, we only need a suitable value of $\gamma$ to compute the estimated $\gamma$-ridge leverage scores ($N$ numbers) and define the sampling distribution for each view.\footnote{The algorithm of~\citet{AlaouiMahoney15a} for estimating the scores does, however, construct an approximation of $\LL_{\gamma}$.}

2. The theorem covers both non-uniform sampling and uniform sampling strategies, through the over-sampling factor $\beta$ (see the remark after Lemma~\ref{lem:sampling}). In practice, we find that uniform sampling tends to work quite well already and ridge leverage scores may bring (moderate) further improvement.

3. There exist interesting trade-offs between accuracy and computation in the algorithm.
\begin{itemize}
\item From the definitions~\eqref{e:leverage-scores} and \eqref{e:deff}, we observe that $\{ l_i (\K,\gamma) \}$ and $d_{\text{eff}}(\K,\gamma)$ are decreasing functions of $\gamma$. As a sanity check, when we require higher approximation accuracy (smaller $\epsilon$), the $d_{\text{eff}}$ terms in the lower bounds of $M_1$ and $M_2$ become larger, indicating that we need to sample more columns. 
\item It is easy to show that $d_{\text{eff}}(\K,\gamma)\le \trace{\K}/N\gamma$. Therefore, for fixed $t_1$ and $t_2$, the necessary ranks $M_1$ and $M_2$ scale as $\calO(1/\epsilon)$. Theoretically, this asymptotic dependence is much better than those obtained for approximate KCCA using data-independent random Fourier features, where ranks of kernel approximations (number of random features) scale as $\calO(1/\epsilon^2)$~\citep[Theorem~4]{Lopez_14b}\footnote{Although their bound is for expected spectral norm error, one can instead obtain a high probability bound when invoking the matrix Bernstein inequality.}.
\end{itemize}

\subsection{Incremental column selection}
In practice, it is computationally expensive to compute exact ridge leverage scores. We can use the efficient algorithm of~\citet[Sec.~3.5]{AlaouiMahoney15a} to compute their approximate values. This would introduce some approximation factor $\beta<1$, and makes it difficult to determine the values of $M_1$ and $M_2$ a priori. Furthermore, in machine learning problems, we also care about the generalization ability of an algorithm, and large ranks that perform well on training data may not be optimal for test data. 

These considerations motivate us to derive an incremental approach where we gradually increase $M_1$ and $M_2$ for the kernel approximations by sampling more columns according to the estimated ridge leverage scores, compute solutions along the path of different ranks, and monitor their performance on validation data to avoid over-fitting.

Note that $\tilde{\T}$ is the multiplication of two matrices, each of the form $(\overline{\LL} + N \lambda \I)^{-1} \overline{\LL}$. By the matrix inversion lemma, % we have
\begin{align}
  (\overline{\LL} + N \lambda \I)^{-1} \overline{\LL}  =\; & \I - N \lambda (\overline{\LL} + N \lambda \I)^{-1} \nonumber \\
  =\; & \I - \left( \I + \frac{1}{N \lambda} \HH \K \SS (\SS^\top \K \SS)^{-1} \SS^\top \K \HH \right)^{-1} \nonumber \\
  =\; & \I - \left( \I - \HH \K \SS \left(  N \lambda \SS^\top \K \SS +  \SS^\top \K \HH \K \SS \right)^{-1}  \SS^\top \K \HH \right) \nonumber \\ \label{e:inversion-lemma}
  =\; & \HH \K \SS \left(  N \lambda \SS^\top \K \SS +  \SS^\top \K \HH \K \SS \right)^{-1}  \SS^\top \K \HH .
\end{align}
Therefore, efficient incremental column selection hinges on inexpensive incremental updates of \eqref{e:inversion-lemma}, where the time complexity mainly comes from computing the inverse.

When we add one more column in the Nystr\"om approximation, the dimension of $\SS^\top \K \SS$ and $\SS^\top \K \HH \K \SS$ both increase by $1$. An efficient way of computing the inverses at each subsampling level is by exploiting rank-one Cholesky updates~\citep{GolubLoan96a}. We give the procedure for computing $\left(  N \lambda \SS^\top \K \SS +  \SS^\top \K \HH \K \SS \right)^{-1}$ in Algorithm~\ref{alg:chol}, where \texttt{cholupdate} is the rank-one update procedure implemented in many standard linear algebra libraries. The algorithm can output the Cholesky decompositions at any intermediate rank and the inverse is computed by efficiently solving two upper/lower triangular systems. A similar algorithm without the centering and scaling operations was used by~\citet{Rudi_15a} for incremental Nystr\"om approximation in kernel ridge regression. 

\begin{algorithm}[t]
  \caption{Incremental algorithm for computing $\left(  N \lambda \SS^\top \K \SS +  \SS^\top \K \HH \K \SS \right)^{-1}$ in Nystr\"om approximation.}
  \label{alg:chol}
  \renewcommand{\algorithmicrequire}{\textbf{Input:}}
  \renewcommand{\algorithmicensure}{\textbf{Output:}}
  \begin{algorithmic}
    \REQUIRE Index set $I=\{i_1,\dots, i_M\}$, regularization $\lambda$.
    \STATE $s_1 \leftarrow \SS_{i_1 1}$,\ \ \ $\aa_1 \leftarrow s_1 \HH [\K_{i_1 1}, \dots, \K_{i_1 N}]^\top$,\ \ \ $\A_1 \leftarrow \aa_1$
    \STATE $d_1 \leftarrow \aa_1^\top \aa_1 + N \lambda s_1^2 \K_{i_1 i_1}$,\ \ \ $\RR_1 \leftarrow \sqrt{d_1}$.
    \FOR {$m=2,\dots,M$}
    \STATE $s_m \leftarrow \SS_{i_m m}$
    \STATE $\aa_m \leftarrow s_m \HH [\K_{i_m 1}, \dots, \K_{i_m N}]^\top$
    \STATE $\A_m \leftarrow [\A_{m-1}, \aa_m]$
    \STATE $\b_m \leftarrow s_m [s_1 \K_{i_m i_1}, \dots, s_{m-1} \K_{i_m i_{m-1}}]^\top$ 
    \STATE $\c_m \leftarrow \A_{m-1}^\top \aa_m + N \lambda \b_m$ 
    \STATE $d_m \leftarrow \aa_m^\top \aa_m + N \lambda s_m^2 K_{i_m i_m}$,\ \ \ $g_m \leftarrow \sqrt{1 + d_m}$
    \STATE $\uu_m \leftarrow \left[\frac{\c_m}{1+g_m};\ g_m \right]$,\ \ \ $\vv_m \leftarrow \left[\frac{\c_m}{1+g_m};\ -1 \right]$
    \STATE $\RR_m \leftarrow \left[\begin{array}{cc} \RR_{m-1} & \0 \\ \0 & \0\end{array} \right]$
    \STATE $\RR_m \leftarrow \mathtt{cholupdate}(\RR_m,\uu_t,\text{'}+\text{'})$
    \STATE $\RR_m \leftarrow \mathtt{cholupdate}(\RR_m,\vv_t,\text{'}-\text{'})$
    \ENDFOR
    \ENSURE $\left(  N \lambda \SS^\top \K \SS +  \SS^\top \K \HH \K \SS \right)^{-1} = \RR_M^{-1} (\RR_M^\top)^{-1} $ where $\RR_M \in \bbR^{M \times M}$ is upper triangular.
  \end{algorithmic}
\end{algorithm}

In NKCCA, our goal is to compute the SVD of $\tilde{\T}=(\overline{\LL}_1 + N \lambda \I)^{-1} \overline{\LL}_1 \overline{\LL}_2 (\overline{\LL}_2 + N \lambda \I)^{-1} \in \bbR^{N \times N}$. Thus we also incrementally compute the QR decompositions $\HH \K_1 \SS_1 = \Q_1 \PP_1$ where $\Q_1\in \bbR^{N \times M_1}$ has orthogonal columns and $\PP_1\in \bbR^{M_1 \times M_1}$ is upper-triangular, and similarly $\HH \K_2 \SS_2 = \Q_2 \PP_2$ using the modified Gram-Schmidt algorithm~\citep{HornJohnson86a}. This allows us to reduce the original SVD problem to computing the SVD of 
\begin{gather*}
  \hat{\T} :=\PP_1 \left(  N \lambda \SS_1^\top \K_1 \SS_1 +  \SS_1^\top \K_1 \HH \K_1 \SS_1 \right)^{-1}  \left( \SS_1^\top \K_1 \HH \cdot \HH \K_2 \SS_2 \right) \left(  N \lambda \SS_2^\top \K_2 \SS_2 +  \SS_2^\top \K_2 \HH \K_2 \SS_2 \right)^{-1} \PP_2^\top 
\end{gather*}
which is of much smaller dimensions ($\hat{\T} \in \bbR^{M_1 \times M_2}$). By computing the SVD of $\hat{\T}=\hat{\U} \hat{\bSigma} \hat{\V}^\top$, we recover the SVD of $\tilde{\T}=(\Q_1 \hat{\U}) \hat{\bSigma} (\Q_2 \hat{\V})^\top$. Then $\tilde{\balpha}'$ (resp. $\tilde{\bbeta}'$) corresponds to the first column of $\Q_1 \hat{\U}$ (resp. $\Q_2 \hat{\V}$).

We incrementally compute the ``core matrix'' $\left( \SS_1^\top \K_1 \HH \cdot \HH  \K_2 \SS_2 \right) \in \bbR^{M_1 \times M_2}$ and save it in memory, and its dimensions grow by $1$ each time we sample one more column. 
Assuming $M_1=M_2=M$, then the time complexity of computing the SVD of $\tilde{\T}$ as described above is $\calO(M^3 + M^2 N)$. Several expensive steps, including forming the core matrix ($\calO(M^2 N)$), the incremental Cholesky decompositions ($\calO(M^3)$), and the modified Gram-Schmidt algorithm ($\calO(M^2 N)$) have computations reused at different subsampling levels, resulting in significant time savings for model selection.

\subsection{Out-of-sample mapping and kernel stability} 

After extracting $\tilde{\T}$'s left sigular vector $\tilde{\balpha}'$ (of unit length), we compute the coefficients $\tilde{\balpha}=\sqrt{N} (\overline{\LL}_1 + N \lambda_1)^{-1} \tilde{\balpha}'$, as approximation to $\balpha=\sqrt{N} (\overline{\K}_1 + N \lambda_1)^{-1} {\balpha}'$,  for combining kernel affinities in the out-of-sample mapping. Since 
\begin{gather*}
  (\overline{\LL}_1 + N \lambda_1)^{-1}=\frac{1}{N\lambda_1} (\I-(\overline{\LL}_1 + N \lambda_1 \I)^{-1} \overline{\LL}_1),
\end{gather*}
we can simply reuse the decomposition of~\eqref{e:inversion-lemma} to compute the inverse.
Furthermore, we can rewrite~\eqref{e:out-of-sample} as
\begin{align}
  f(\x) & = \sum\nolimits_{i=1}^N \alpha_i \left< \bar{\phi}(\x_i),\phi(\x) - \frac{1}{N} \sum\nolimits_{j=1}^N \phi(\x_j) \right> \nonumber \\
  & = \sum\nolimits_{i=1}^N \alpha_i \left< \bar{\phi}(\x_i),\phi(\x) \right> + \text{const}  \nonumber \\ \label{e:out-of-sample-rewrite}
  & = \mathbf{k}^\top \HH \balpha + \text{const}
\end{align}
where the constant is independent of $\x$ (thus need not be computed), and 
$\mathbf{k} = \left[k_x(\x_1,\x),\dots, k_x(\x_N,\x)\right]^\top$.  

Assuming the exact kernel function is used during testing, the out-of-sample mapping for NKCCA is (ignoring the same constant in \eqref{e:out-of-sample-rewrite})
\begin{gather}
  \tilde{f}(\x):=\mathbf{k}^\top \HH \tilde{\balpha}.
\end{gather}

We now study the perturbation in out-of-sample mapping resulting from the low-rank kernel approximations used in training NKCCA, referred to as ``kernel stability'' by~\citet{Cortes_10b}. % As in their work, we assume that the exact kernel function is used during testing. 
The following theorem provides the kernel stability when there exists a non-zero singular value gap for $\T$ (used by exact KCCA). The condition on the kernel affinity being bounded, \ie, $k_x (\x,\x') \le c$ for all $\x$ and $\x'$, is verified with $c=1$ for the Gaussian RBF kernel $k_x (\x,\x')=e^{-\frac{\norm{\x-\x'}^2}{2\sigma^2}}$ for example.

\begin{thm} \label{thm:stability}
  Use the sampling strategy stated in Theorem~\ref{thm:guarantee} for NKCCA, so that $\norm{\T-\tilde{\T}} \le \epsilon$ with probability at least $1 - \delta$.
  Furthermore, assume the following two conditions
  \begin{enumerate}
  \item \label{cond:bound} $\forall \x, \x'$, we have $k_x (\x,\x') \le c$.
  \item \label{cond:gap} There exist $r>0$ such that $\sigma_1 (\T) - \sigma_2 (\T) \ge r$, and $\norm{ \T-\tilde{\T} } \le \frac{r}{2}$.
  \end{enumerate}
  Then with the same probability, we have
  \begin{gather*}
    \abs{f(\x)-\tilde{f}(\x)} \le \left( \frac{1}{2} + \frac{4\sqrt{2}}{r} \right) \frac{c \epsilon}{\lambda_1} \qquad \forall \x.
  \end{gather*}
\end{thm}
\begin{proof} The proof consists of three steps.

  \textbf{Step I.} We first bound the perturbation in the unit singular vectors, \ie, $\norm{\balpha'-\tilde{\balpha}'}$. % This is achieved by bounding the perturbation in the concatenated vectors $\norm{ [\balpha'; \bbeta'] - [\tilde{\balpha}'; \tilde{\bbeta}']}$. 

  Let the complete set of left singular vectors of $\T$ be $\{\uu_i\}_{i=1}^N$ (with $\uu_1=\balpha'$), and the complete set of right singular vectors be $\{\vv_i\}_{i=1}^N$ (with $\vv_1=\bbeta'$), both of which constitute orthonormal basis sets of $\bbR^N$.
  Observe that the set of concatenated vectors $\{\w_i\}_{i=1}^{2N}$ with $\w_i=[\uu_i; \vv_i]$ for $i=1,\dots,N$, and $\w_{N+i}=[\uu_{N-i+1}; -\vv_{N-i+1}]$ for $i=1,\dots,N$, are eigenvectors of 
  % \begin{gather*}
  $    \C:=\left[
    \begin{array}{cc}
      \0 & \T \\
      \T^\top & \0
    \end{array}\right] $,
  % \end{gather*}
  with corresponding eigenvalues (in descreasing order)
  \begin{gather*}
    \theta_1=\sigma_1 (\T), \quad \dots, \quad \theta_N=\sigma_N (\T), \\
    \theta_{N+1}  = - \sigma_N (\T), \quad \dots, \quad \theta_{2N}= - \sigma_1 (\T).
  \end{gather*}
  Define also
  % \begin{gather*}
  $    \tilde{\C}:=\left[
    \begin{array}{cc}
      \0 & \tilde{\T} \\
      \tilde{\T}^\top & \0
    \end{array}\right] $
  % \end{gather*}
  and its set of eigenvectors $\{\tilde{\w}_i\}_{i=1}^{2N}$ similarly, with $\tilde{\w}_1=[\tilde{\balpha}'; \tilde{\bbeta}']$, and corresponding eigenvalues $\{ \tilde{\theta}_i \}_{i=1}^{2N}$. The rest of this step is similar to~\citet[Prop.~2]{SmaleZhou09a}.

  Note that $\{\w_i\}_{i=1}^{2N}$ constitute an orthogonal basis of $\bbR^{2N}$. Thus, we can write $\tilde{\w}_1$ as a linear combination of this basis: 
  \begin{gather*}
    \tilde{\w}_1=a_1 \w_1 + \dots + a_{2N} \w_{2N},
  \end{gather*}
  with $\sum_{i=1}^{2N} a_i^2 = 1$ since all vectors in the equality have equal length of $\sqrt{2}$. 

  Consider $\tilde{\C} \tilde{\w}_1 - \C \tilde{\w}_1$. It can be written as
  \begin{align*}
    \tilde{\C} \tilde{\w}_1 - \C \tilde{\w}_1 & = \tilde{\theta}_1 \tilde{\w}_1 - \sum_{i=1}^{2N} a_i \C \w_i  = \sum_{i=1}^{2N}  (\tilde{\theta}_1 - \theta_i) a_i \w_i.
  \end{align*}
  For $i>1$, according to condition~\ref{cond:gap}, we have 
  \begin{align*}
    \abs{ \tilde{\theta}_1 - \theta_i } & \ge \abs{ \theta_1 - \theta_i } - \abs{  \theta_1 - \tilde{\theta}_1 } \ge \abs{ \theta_1 - \theta_2 } - \abs{  \theta_1 - \tilde{\theta}_1 } \ge \frac{r}{2}.
  \end{align*}
  As a result,
  \begin{align} \label{e:cw}
    \norm{ \tilde{\C} \tilde{\w}_1 - \C \tilde{\w}_1 }^2 = 2 \sum_{i=1}^{2N} ( \tilde{\theta}_1 - \theta_i )^2 a_i^2  \ge 2 \sum_{i=2}^{2N} ( \tilde{\theta}_1 - \theta_i )^2 a_i^2 \ge \frac{r^2}{2} \sum_{i=2}^{2N} a_i^2 .
  \end{align}
  On the other hand, we must have
  \begin{gather*}
    \norm{ \tilde{\C} \tilde{\w}_1 - \C \tilde{\w}_1 }^2 \le \norm{ \tilde{\C} - \C }^2 \norm{\tilde{\w}_1}^2 = {2} \norm{ \tilde{\C} - \C }^2.
  \end{gather*}
  Also, $\norm{\tilde{\C}-\C}=\norm{\tilde{\T}-\T}$ %= \max \left\{ \norm{\tilde{\T}-\T}, \norm{\tilde{\T}^\top-\T^\top} \right\}
  due to the block structure of $\tilde{\C}$ and $\C$. Thus from \eqref{e:cw} we obtain
  \begin{gather} \label{e:partial-sum-squares}
    \sum_{i=2}^{2N} a_i^2  \le \frac{4}{r^2} \norm{\tilde{\T}-\T}^2.
  \end{gather}
  In view of the inequality $1-x \le \sqrt{1-x^2}$ for $x \in [0,1]$, we also have\footnote{Without loss of generality, we assume $a_1=\tilde{\w}_1^\top \w_1 /2 \ge 0$. Otherwise, we can set $\tilde{\w}=[-\tilde{\balpha}'; -\tilde{\bbeta}']$ as $-\tilde{\balpha}'$ and $-\tilde{\bbeta}'$ are also a valid left/right singular vector pair of $\tilde{\T}$.} 
  \begin{gather}\label{e:a1}
    1 - a_1 \le \sqrt{1-a_1^2} = \sqrt{\sum_{i=2}^{2N} a_i^2}.
  \end{gather}

  Now we can bound the perturbation
  \begin{align*}
    \norm{\balpha'-\tilde{\balpha}'} & \le \norm{\w_1-\tilde{\w}_1} \\
    & = \norm{ (1-a_1) \w_1 - \sum_{i=2}^{2N} a_i \w_i} \\
    & \le \norm{(1-a_1) \w_1} +  \norm{\sum_{i=2}^{2N} a_i \w_i} \\
    & = \sqrt{2} \left( \abs{1-a_1} + \sqrt{\sum_{i=2}^{2N} a_i^2}  \right) \\
    & \le 2\sqrt{2}  \sqrt{\sum_{i=2}^{2N} a_i^2} \le \frac{4\sqrt{2}}{r} \norm{\tilde{\T}-\T},
  \end{align*}
  where we used that $\balpha'-\tilde{\balpha}'$ is a subvector of $\w_1-\tilde{\w}_1$, the triangle inequality, \eqref{e:a1} and \eqref{e:partial-sum-squares} in the four inequalities.

  \textbf{Step II.} We then bound the perturbation in the actual coefficients in $f$ and $\tilde{f}$, \ie, $\norm{\balpha-\tilde{\balpha}}$. We have
  \begin{align*}
    \frac{\norm{\balpha-\tilde{\balpha}}}{\sqrt{N}} &= \norm{ (\overline{\K}_1 + N \lambda_1 \I)^{-1} \balpha' - (\overline{\LL}_1 + N \lambda_1 \I)^{-1} \tilde{\balpha}' } \\
    & \le \norm{(\overline{\K}_1 + N \lambda_1 \I)^{-1} (\balpha'-\tilde{\balpha}')} + \norm{ \left( (\overline{\K}_1 + N \lambda_1 \I)^{-1} -  (\overline{\LL}_1 + N \lambda_1 \I)^{-1} \right) \tilde{\balpha}' } \\
    & \le \norm{(\overline{\K}_1 + N \lambda_1 \I)^{-1}} \norm{ \balpha'-\tilde{\balpha}' } + \norm{ (\overline{\K}_1 + N \lambda_1 \I)^{-1} -  (\overline{\LL}_1 + N \lambda_1 \I)^{-1} } \norm{\tilde{\balpha}'}.
  \end{align*}
  where the triangle inequality is used in the first inequality.
  
We have already shown that
  \begin{gather*}
    N \lambda_1 \norm{ (\overline{\K}_1 + N \lambda_1 \I)^{-1} -  (\overline{\LL}_1 + N \lambda_1 \I)^{-1} } = \norm{ \overline{\K}_1 (\overline{\K}_1 + N \lambda_1 \I)^{-1} - \overline{\LL}_1 (\overline{\LL}_1 + N \lambda_1 \I)^{-1} } \le \epsilon/2
  \end{gather*}
  in Theorem~\ref{thm:guarantee} using the stated sampling strategy. 

  Using the facts $\norm{ (\overline{\K}_1 + N \lambda_1 \I)^{-1} } \le \frac{1}{N \lambda_1}$, $\norm{ \tilde{\balpha}' }=1$, and the bound of $\norm{\balpha'-\tilde{\balpha}'}$ from Step I, we have
  \begin{align*}
    \frac{\norm{\balpha-\tilde{\balpha}}}{\sqrt{N}} \le \frac{1}{N \lambda_1} \frac{4\sqrt{2}}{r} \epsilon + \frac{1}{N \lambda_1} \frac{\epsilon}{2} = \left( \frac{1}{2} + \frac{4\sqrt{2}}{r} \right) \frac{\epsilon}{N \lambda_1}.
  \end{align*}

  \textbf{Step III.} We now bound the kernel stability. By the Cauchy-Schwarz inequality, we have 
  \begin{align*}
    \abs{ f(\x) - \tilde{f}(\x) } & = \abs{ \mathbf{k}^\top \HH (\balpha-\tilde{\balpha}) } \le \norm{ \mathbf{k} } \norm{ \HH (\balpha-\tilde{\balpha}) } \le \norm{ \mathbf{k} } \norm{ \balpha-\tilde{\balpha} },
  \end{align*}
  where we have used $\norm{\HH}=1$ in the last inequality.
  Since $\mathbf{k}$ is an $N$-dimensional vector and each entry is bounded by $c$ (condition~\ref{cond:bound}), we have $\norm{\mathbf{k}} \le \sqrt{N} c$. Combining the bound of $\norm{ \balpha-\tilde{\balpha} }$ from Step II, we conclude the proof.
\end{proof}

\paragraph{Remarks} Theorem~\ref{thm:stability} can be extended to the case where $L$-dimensional projections are sought for $L>1$, by assuming a gap at the $L$-th singular value of $\T$. Our proof technique is general and works for other low-rank KCCA algorithms, if only a bound on $\norm{\T - \tilde{\T}}$ is available. With Theorem~\ref{thm:stability}, one may show the consistency of NKCCA, by controlling kernel stability to be much smaller than the generalization bound of exact KCCA~\citep{Fukumiz_07a}. 

We note in passing that, unlike kernel ridge regression~\citep{Rudi_15a}, KCCA does not have a representer theorem where $f$ is restricted to the selected landmarks, as the centering operation in RKHS already involves all training samples. When the training set is large, it is worth exploring approximation strategies (see, \eg, \citealp{Hsieh_14a} and the references therein) to speed up prediction.

\section{Related work}
\label{s:related}

There is a rich literature for the Nystr\"om method, exploring the sampling strategies and approximation quality\linebreak[4] \citep{DrineasMahoney05a,Zhang_08e,Gitten11a,Farahat_11a,Kumar_12a,WangZhang13a,GittenMahoney13a}, applications for large-scale kernel machines~\citep{WilliamSeeger00a,Fowlkes_04a,Platt05a,ZhangKwok10a,Li_10b}, and generalization performance of kernel machines with Nystr\"om kernel approximations~\citep{Cortes_10b,Yang_12e,Jin_13a,Bach13d,AlaouiMahoney15a,Rudi_15a}.

The ridge leverage scores are closely related to the \emph{leverage scores relative to the best rank-$k$ space}\linebreak[4] \citep{GittenMahoney13a}, as the $\frac{\sigma_j (\K)}{\sigma_j (\K) + N \gamma}$ terms in \eqref{e:leverage-scores} implement ``soft shrinkage'' of the spectrum of $\K$. In general, leverage scores measure the extent to which each sample ``stands out'', and are useful in a wide range of fields~\citep{Mahoney11a}.

\citet{Avron_13a} and~\citet{Paul15a} have recently proposed randomized sample selection algorithms to scale up linear CCA. Their goal is to select a subset of \emph{paired} examples, whose canonical correlation well approximates that of the entire training set. The algorithm of~\citet{Avron_13a} first homogenizes the importance of each sample via a structured random projection of the training set and then uses uniform sampling on top; the algorithm of~\citet{Paul15a} uses leverage scores of certain (carefully constructed) matrix for sampling. Note that in KCCA, we are not constrained to sample data pairs; we can approximate the kernel of each view using independent columns. On the other hand, theoretically, there is little distinction between random projection and column selection: The matrices $\SS_1$ and $\SS_2$ can be implementing random projections instead of random column selection, as long as they possess certain spectral properties~\citep{GittenMahoney13a}.

Another approach to large-scale kernel machines is to use random Fourier features~\citep{RahimiRecht08a,RahimiRecht09a}. \citet{Lopez_14b} map original inputs of each view to high-dimensional random feature spaces and then run linear CCA on top to approximate KCCA. While random Fourier features are data-independent and efficient to generate, the performance of this approach tends to be worse than that of the Nystr\"om method for the same approximation rank~\citep{Lopez_14b,Wang_15b,Xie_15a}. A key difference between the two approaches is that random Fourier features are designed to approximate the kernel functions, while the Nystr\"om method aims to approximate the kernel matrix~\citep{Yang_12e}.

\section{Experiments}
\label{s:expt}

\begin{figure}[t]
\centering
\psfrag{scores}[bc][c][1.2]{scores}
\psfrag{N}[c][c][.6]{}
\psfrag{error}[bc][c][1.2]{error}
\psfrag{M}[c][c][1.0]{$M_1=M_2$}
\begin{tabular}{@{}c@{\hspace{0.05\linewidth}}c@{}}
\psfrag{gamma0, view 1}[l][l][0.9]{$\gamma=\lambda\times 10^{0}$, view 1}
\psfrag{gamma0, view 2}[l][l][0.9]{$\gamma=\lambda\times 10^{0}$, view 2}
\psfrag{gamma1, view 1}[l][l][0.9]{$\gamma=\lambda\times 10^{-1}$, view 1}
\psfrag{gamma1, view 2}[l][l][0.9]{$\gamma=\lambda\times 10^{-1}$, view 2}
\psfrag{gamma2, view 1}[l][l][0.9]{$\gamma=\lambda\times 10^{-2}$, view 1}
\psfrag{gamma2, view 2}[l][l][0.9]{$\gamma=\lambda\times 10^{-2}$, view 2}
\psfrag{gamma3, view 1}[l][l][0.9]{$\gamma=\lambda\times 10^{-3}$, view 1}
\psfrag{gamma3, view 2}[l][l][0.9]{$\gamma=\lambda\times 10^{-3}$, view 2}
\psfrag{gamma4, view 1}[l][l][0.9]{$\gamma=\lambda\times 10^{-4}$, view 1}
\psfrag{gamma4, view 2}[l][l][0.9]{$\gamma=\lambda\times 10^{-4}$, view 2}
\includegraphics[width=0.45\linewidth]{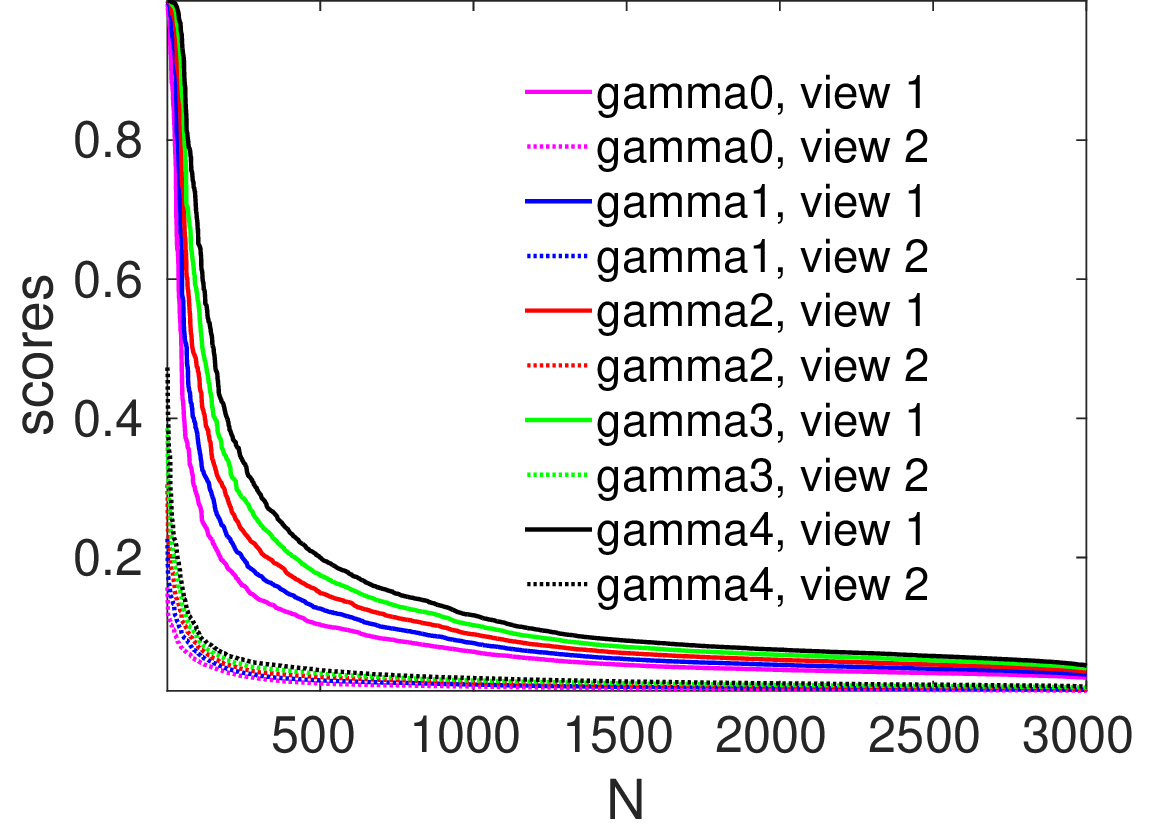}
&
\psfrag{uniform rho-rho}[l][l][0.9]{$\abs{ \rho - \tilde{\rho} }$, uniform}
\psfrag{ridge rho-rho}[l][l][0.9]{$\abs{ \rho - \tilde{\rho} }$, ridge}
\psfrag{uniform T-T}[l][l][0.9]{$\norm{ \T - \tilde{\T} }$, uniform}
\psfrag{ridge T-T}[l][l][0.9]{$\norm{ \T - \tilde{\T} }$, ridge}
\includegraphics[width=0.45\linewidth]{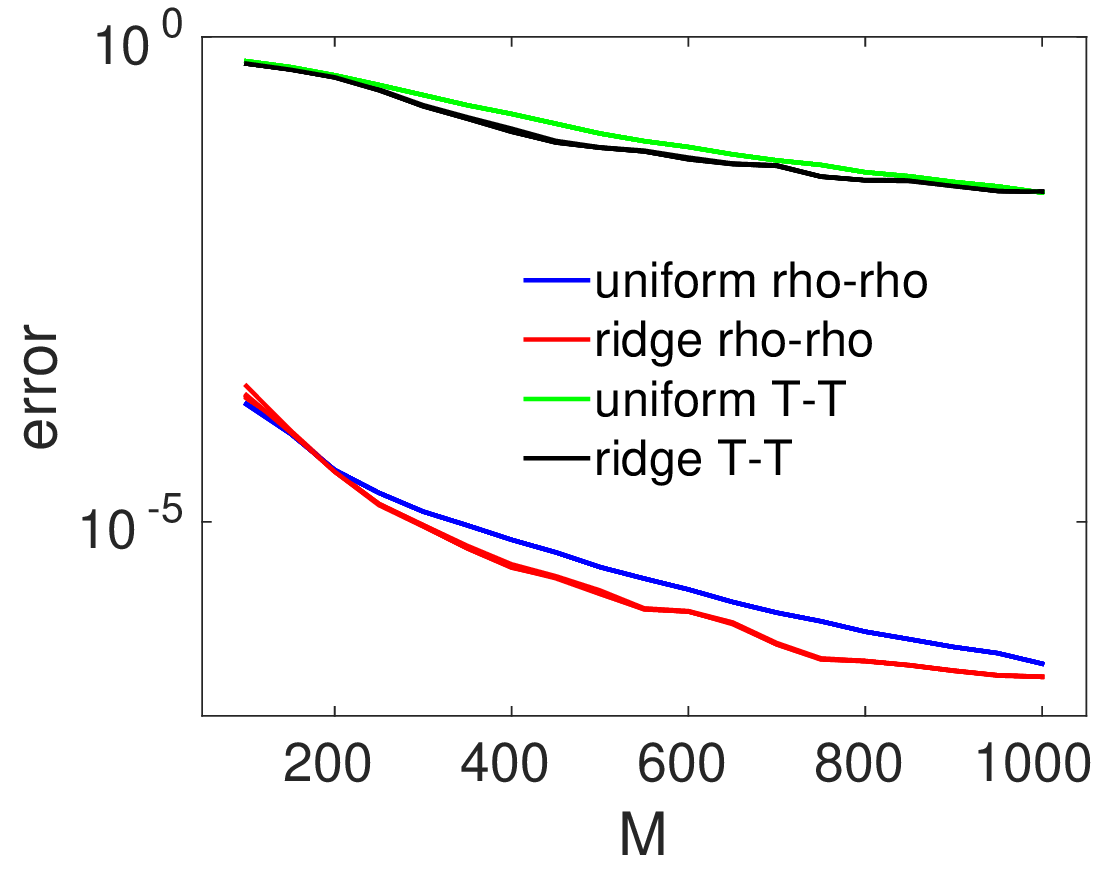} \\
(a) & (b) \\[3ex]
\psfrag{uniform alpha}[l][l][0.9]{$\norm{ \balpha - \tilde{\balpha} }/\sqrt{N}$, uniform}
\psfrag{ridge alpha}[l][l][0.9]{$\norm{ \balpha - \tilde{\balpha} }/\sqrt{N}$, ridge}
\psfrag{uniform T-T}[l][l][0.9]{bound, uniform}
\psfrag{ridge T-T}[l][l][0.9]{bound, ridge}
\includegraphics[width=0.45\linewidth]{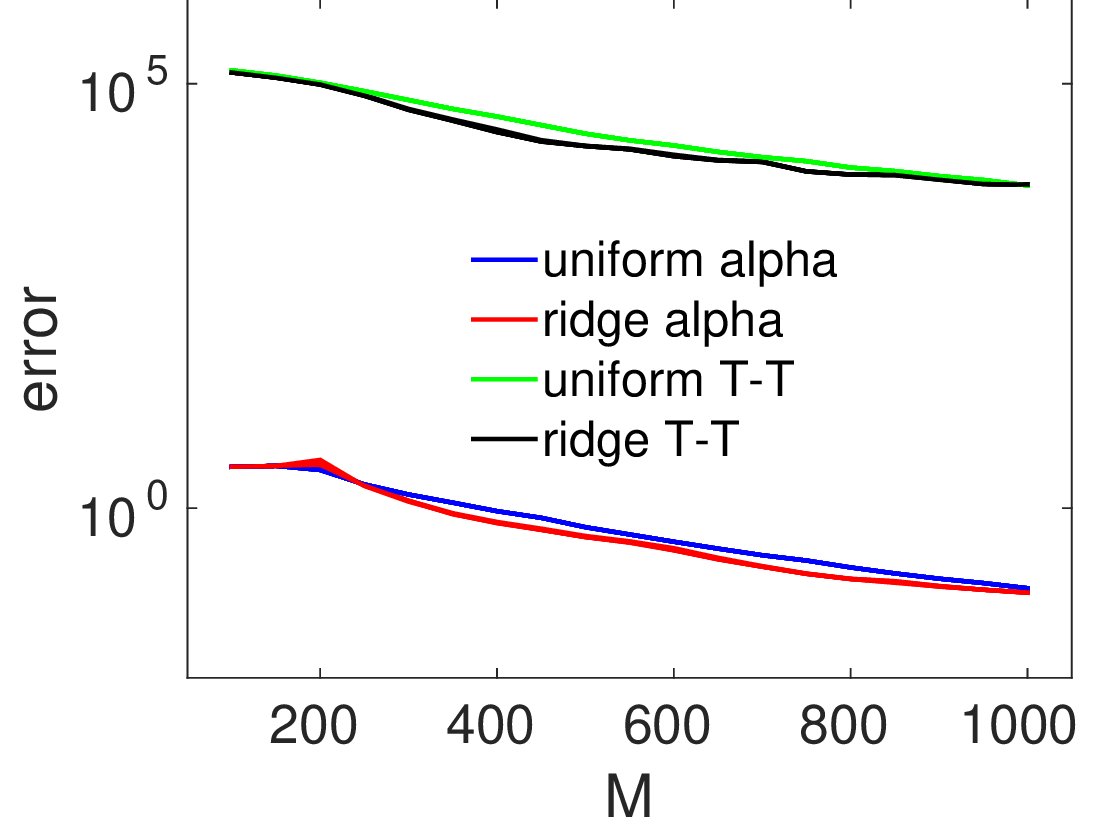} &
\psfrag{speedup}[bl][l][1.2]{speedup}
\includegraphics[width=0.45\linewidth]{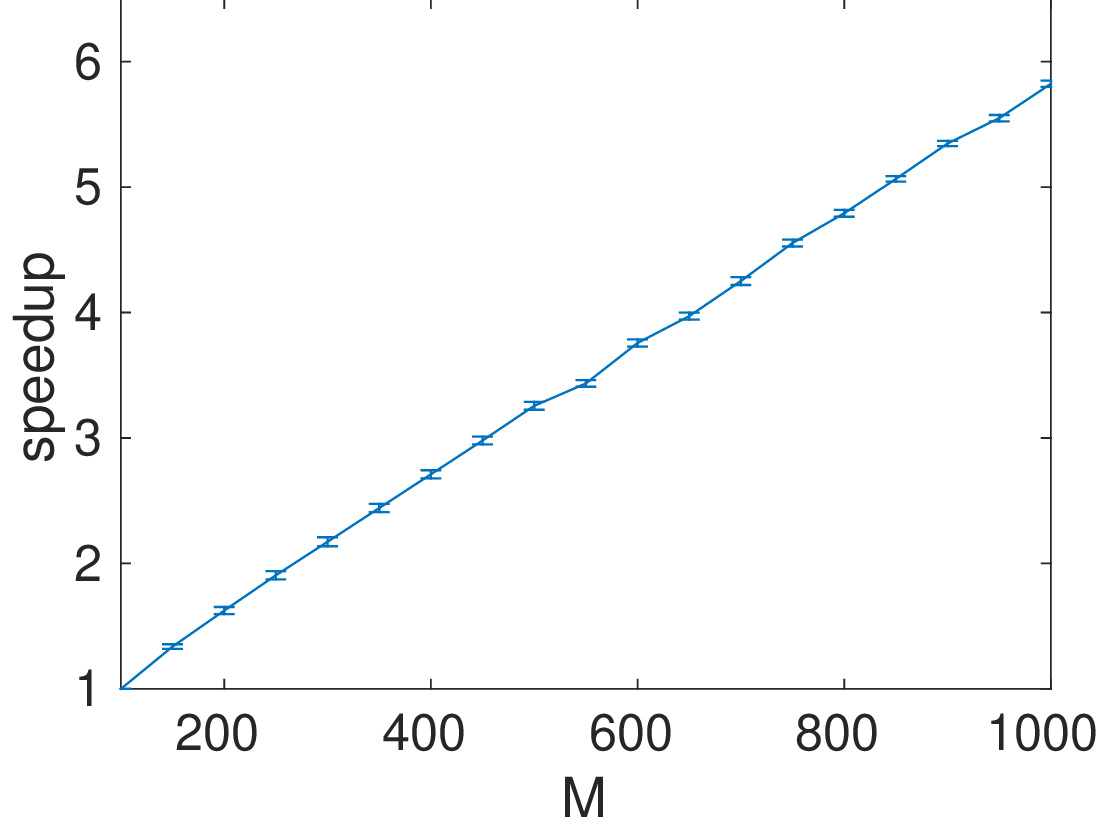}
\\
(c) & (d)
\end{tabular}
\caption{Results on synthetic data. (a) $\gamma$-ridge leverage scores (in sorted order) for each view. (b) $\abs{ \rho - \tilde{\rho} }$ and $\norm{ \T - \tilde{\T} }$. (c) $\frac{\norm{\balpha-\tilde{\balpha}}}{\sqrt{N}}$ vs. the bound $\left( \frac{1}{2} + \frac{4\sqrt{2}}{r} \right) \frac{\norm{\T-\tilde{\T}}}{N \lambda_1}$. (d) Speedups achieved by our incremental algorithm at intermediate ranks. Curves in (b)(c)(d) show averaged results over 20 random seeds.}
\label{f:synthetic}
\end{figure}

\paragraph{Synthetic dataset} We first verify our theoretical results on the synthetic dataset used by~\citet{Fukumiz_07a}. The input data are generated as follows
\begin{gather*}
Z_i \sim \text{Uniform}[0,1], \\
\eta^x_i \sim \calN(0,\, 0.02), \quad U_i=Z_i+0.06+\eta^x_i,\\
\eta^y_i \sim \calN(0,\, 0.03), \quad  V_i=Z_i+3+\eta^y_i, \\
R^x_i = \sqrt{-4 \log (U_i/1.5)}, \quad  \theta^x_i \sim \text{Uniform}[0,2\pi], \\
R^y_i = \sqrt{-4 \log (V_i/4.1)}, \quad  \theta^y_i \sim \text{Uniform}[0,2\pi], \\
\x_i=[R^x_i \cos \theta^x_i, R^x_i \sin \theta^x_i]^\top,\ \y_i=[R^y_i \cos \theta^y_i, R^y_i \sin \theta^y_i]^\top.
\end{gather*}
We sample $N=3000$ points as the training set and another $3000$ points as the tuning set. We perform exact KCCA and NKCCA with both uniform sampling and ridge leverage scores sampling on the training set to compute the first canonical correlation, and compare their approximation error $\abs{ \tilde{\rho} -\rho }$. Gaussian kernel widths and regularization parameters $(\lambda_1, \lambda_2)$ are selected for best canonical correlation on the tuning set with KCCA. Exact $\gamma$-ridge leverage scores are computed at $\gamma=\lambda\times \{10^{-4}, 10^{-3}, 10^{-2}, 10^{-1}, 1\}$ for both views and used to compute the sampling probabilities~\eqref{e:sampling}, and they all lead to very similar results on this dataset. We plot the $\gamma$-ridge leverage scores (in sorted order) for each view in Figure~\ref{f:synthetic}(a). Observe that the scores for view 1 are less uniformly distributed. We show the approximation error for a wide range of ranks ($M_1=M_2$) from 100 to 1000 in Figure~\ref{f:synthetic}(b), and show in (c) the kernel stability and specifically the result of Theorem~\ref{thm:stability} step \textbf{II} (bound in step \textbf{III} can be loose), \ie, $\frac{\norm{\balpha-\tilde{\balpha}}}{\sqrt{N}}$ vs. $\left( \frac{1}{2} + \frac{4\sqrt{2}}{r} \right) \frac{\norm{\T-\tilde{\T}}}{N \lambda_1}$. We see that the actual approximation errors in canonical correlation and kernel stability are smaller than the theoretical bounds, with non-uniform sampling performing somewhat better than uniform sampling.

To demonstrate the efficiency of our incremental Nystr\"om approximation algorithm for NKCCA, we run it for up to $M_1=M_2=1000$ and output the solutions at ranks $\{100, 150, 200, \dots, 1000\}$. At each intermediate rank, We compare the run time of the incremental approach with the total run time of the non-incremental algorithm (starting fresh for each smaller rank), and plot the speedups in Figure~\ref{f:synthetic}(d). It is clear that the speedup improves as the number of intermediate test points increases.

\vspace*{-.5ex}
\paragraph{Acoustic-articulatory data} We next experiment with a subset of the University of Wisconsin X-Ray Microbeam corpus ~\cite{Westbur94a}, which consists of simultaneously recorded acoustic and articulatory measurements during speech. The acoustic view inputs are 39D mel-frequency cepstral coefficients and the articulatory view inputs are horizontal/vertical displacements of 8 pellets attached to different parts of the vocal tract, each then concatenated over a 7-frame context window. As in~\cite{Lopez_14b}, we reshuffle all frames of speaker `JW11', and split them into 30K/10K/10K frames for training/tuning/test. The projection dimensionality $L$ is set to 112, which is also the maximum possible total canonical correlation. 

We compare the randomized KCCA (RCCA) algorithm of~\cite{Lopez_14b} using random Fourier features, NKCCA with uniform sampling, and NKCCA with ridge leverage scores sampling. Gaussian kernel widths and regularization parameters are selected on the tuning set. We compute approximate $\gamma$-ridge leverage scores using the algorithm in~\citet[Sec.~3.5]{AlaouiMahoney15a} with $5000$ randomly sampled columns. $\gamma$ is tuned over $\lambda\times \{10^{-2}, 10^{-1}, 1, 10^1\}$ and $\gamma=\lambda$ is chosen. Total canonical correlations achieved by different algorithms on the test set, averaged over $5$ random seeds, are reported in Table~\ref{t:jw11} for several ranks $(M_1=M_2)$. We observe that NKCCAs outperform RCCA with the same approximation rank, which is consistent with previous work, and sampling based on approximate ridge leverage scores improves over uniform sampling by a small margin.

\begin{table}[t]
\centering
\caption{Total canonical correlations ($\uparrow$) achieved by different approximate KCCA algorithms on JW11 test set with various ranks.}
\label{t:jw11}
\begin{tabular}{|c||r|r|r|r|r|r|}\hline
rank $M_1=M_2$& $1000$  & 2000 & 3000 & 4000 & 5000 \\ \hline
RCCA          & 69.4    & 82.6 & 88.8 & 92.9 & 95.8 \\ \hline
NKCCA unif.    & 78.5    & 89.4 & 94.6 & 98.4 & 100.8 \\ \hline
NKCCA ridge    & \textbf{79.6}    & \textbf{90.5} & \textbf{95.7} & \textbf{99.2} & \textbf{101.5} \\ \hline
\end{tabular}
\end{table}

\vspace*{-.5ex}
\paragraph{Discussion} Empirically, we find that non-uniform sampling based on ridge leverage scores provides moderate improvement over uniform sampling. Several factors are at play. First, ridge leverage scores are data-dependent quantities, and theoretically there is a clear advantage to use non-uniform sampling if the ridge leverage scores are quite non-uniform. Second, for large datasets, it is computationally infeasible to compute exact ridge leverage scores, and approximating them results in sub-optimal performance. Third, although ridge leverage scores provide better approximation for kernel ridge regression where a single kernel is involved, the interplay between the two kernels is also crucial in KCCA. This interplay is not fully taken into account by our analysis (\eg, the upper bound in~\eqref{e:error-two-views} consists of error terms for each view individually) and our sampling strategy. It is an interesting future direction to study sampling strategies that use information from the other view, which may lead to improved approximation accuracy.
\vspace*{-1ex}
\section{Conclusions}
\label{s:conclusion}

We have proposed a non-uniform sampling strategy for column selection in approximate KCCA using the Nystr\"om method, proved its approximation guarantees for both training (error in canonical correlation) and testing (kernel stability), and also developed an algorithm for computing solutions incrementally at different approximation ranks.

It has been shown by~\citet{Bach13d,AlaouiMahoney15a} that for kernel ridge regression, which has a closed-form solution (as does KCCA), it is possible to derive a sharper generalization bound directly using the solution with low-rank kernel approximations, rather than indirectly through the kernel stability~\citep{Cortes_10b}. Therefore, another interesting problem is to develop a sharp statistical performance guarantee of NKCCA based on known results for exact KCCA~\citep{Fukumiz_07a}.

%% our own large scale stochastic KCCA work in related work

\appendix

\section{Proof of Lemma~\ref{lem:error}}
 \begin{proof} We first prove the bound for uncentered kernel matrices. Notice that $\K (\K + N \lambda \I)^{-1}=\I - N \lambda  (\K + N \lambda \I)^{-1}$. Therefore, for any $\gamma \ge 0$,
   \begin{align}
     & \norm{ \K (\K + N \lambda \I)^{-1} - \LL_{\gamma} (\LL_{\gamma} + N \lambda \I)^{-1} } \nonumber \\
     =\; & \norm{ \left( \I - N \lambda (\K + N \lambda \I)^{-1} \right)  - \left( \I - N \lambda (\LL_{\gamma} + N \lambda \I)^{-1} \right) } \nonumber \\
     =\; & N \lambda \norm{  (\K + N \lambda \I)^{-1}  -  (\LL_{\gamma} + N \lambda \I)^{-1} } \nonumber \\
     =\; &  N \lambda \norm{ (\K + N \lambda \I)^{-1} ( \LL_{\gamma} - \K ) (\LL_{\gamma} + N \lambda \I)^{-1} } \nonumber \\
     \le \; & N \lambda \norm{ (\K + N \lambda \I)^{-1}} \norm{ \K - \LL_{\gamma} } \norm{ (\LL_{\gamma} + N \lambda \I)^{-1} } \nonumber \\ \label{e:lemma-error-1}
     \le \; & \frac{1}{N \lambda} \norm{ \K - \LL_{\gamma} } 
   \end{align}
   where we have used the fact $\norm{(\K + N\lambda \I)^{-1}}\le \frac{1}{N\lambda}$ in the last inequality.

   For $\gamma>0$, we can use $\K - \LL_{\gamma} \preceq \frac{N\gamma}{1-t} \I$ from Lemma~1 to obtain the desired bound. On the other hand, we observe that $\K - \LL \preceq \K - \LL_{\gamma}$ for any $\gamma>0$ owing to $\LL_{\gamma} \le \LL$ from Lemma~1. Therefore $\norm{\K - \LL} \le \norm{\K - \LL_{\gamma}}$ and the same bound holds for $\LL$.

   Now, for centered matrices, a derivation similar to that of \eqref{e:lemma-error-1} follows and we just need to bound $\norm{\overline{\K}-\overline{\LL}_\gamma}$. By definitions of the centered matrices and conjugating the PSD ordering~\citep[Obs. 7.7.2]{HornJohnson86a}, we have
   \begin{gather}
     \overline{\K} - \overline{\LL} \preceq \overline{\K} - \overline{\LL}_{\gamma} \preceq \frac{N\gamma}{1-t} \HH^2 \preceq \frac{N\gamma}{1-t} \I
   \end{gather}
   where the last step is due to the fact that eigenvalues of $\HH$ are in $\{0,1\}$. Thus the reasoning for uncentered matrices also holds here.
 \end{proof}

\bibliographystyle{abbrvnat}
\bibliography{icml16a}

\end{document}